\documentclass[letterpaper, 10 pt, conference]{ieeeconf}
\IEEEoverridecommandlockouts

\usepackage{amssymb}
\usepackage{amsmath}
\usepackage{mathtools}
\usepackage{dsfont}

\usepackage{cite}
\usepackage[ruled,linesnumbered,noline,noend]{algorithm2e}
\usepackage{float}
\usepackage{textcomp}
\usepackage{bm}

\newtheorem{lemma}{Lemma}
\newtheorem{definition}{Definition}

\newtheorem{assum}{Assumption}

\newtheorem{remark}{Remark}

\usepackage{makecell}

\usepackage{graphicx}
\usepackage{tikz}
\usepackage{float}
\usetikzlibrary{arrows}
\usetikzlibrary{arrows.meta}
\usetikzlibrary{arrows,positioning} 
\usetikzlibrary{shapes}
\usetikzlibrary{calc}
\usetikzlibrary{fadings}
\usetikzlibrary{decorations.pathreplacing}
\usetikzlibrary{decorations.text}

\usepackage{hyperref}
\usepackage{xcolor}
\hypersetup{
    colorlinks,
    linkcolor={black},
    citecolor={blue!50!black},
    urlcolor={blue!80!black}
}
\usepackage[normalem]{ulem}

\usepackage{subfigure}

\usepackage{amsfonts}

\usepackage{scalerel}
\def\msquare{\mathord{\scalerel*{\Box}{gt}}}
\def\mdiamond{\mathord{\scalerel*{\Diamond}{gt}}}
\makeatletter \providecommand\dotdiamond{\mathpalette\@barred\mdiamond} \def\@barred#1#2{\ooalign{\hfil$#1\cdot$\hfil\cr\hfil$#1#2$\hfil\cr}}  \makeatother
\makeatletter \providecommand\dotbox{\mathpalette\@burrow\msquare} \def\@burrow#1#2{\ooalign{\hfil$#1\cdot$\hfil\cr\hfil$#1#2$\hfil\cr}}  \makeatother

\usepackage[top=54pt, left=54pt, right=54pt, bottom=54pt]{geometry}
\newcommand{\state}{\textbf{x}}
\newcommand{\ctrl}{\textbf{u}}
\newcommand{\disturb}{w}
\newcommand{\sagx}{x}
\newcommand{\local}{^{\rm loc}}

\newcommand{\nextq}{_{q+1}}
\newcommand{\currq}{_{q}}

\newcommand{\ith}{^{\rm (i)}}
\newcommand{\jth}{^{\rm (j)}}
\newcommand{\inp}{\bm{\xi}}
\newcommand{\outp}{\bm{z}}

\begin{document}

\title{\textbf{Probabilistically-Safe Bipedal Navigation over Uncertain Terrain via Conformal Prediction and Contraction Analysis}}


\author{Kasidit~Muenprasitivej*, Jesse~Jiang*, Abdulaziz~Shamsah*, Samuel~Coogan, and Ye~Zhao 
\thanks{This work was supported in part by a National Science Foundation Graduate Research Fellowship under grant \#DGE-2039655.}
\thanks{The authors are with the Institute of Robotics and Intelligent Machines, Georgia Institute of Technology, Atlanta, GA 30332, USA (email: \{kmuenpra3, chou\}@gatech.edu, ye.zhao@me.gatech.edu). 
}}
\author{Kasidit~Muenprasitivej$^{1}$, Ye~Zhao$^{2}$, and Glen Chou$^{1,3}$
\thanks{$^{1}$Daniel Guggenheim School of Aerospace Engineering, Georgia Institute of Technology, Atlanta, GA, USA (e-mail: kmuenpra3@gatech.edu)}
\thanks{$^{2}$George W. Woodruff School of Mechanical Engineering, Georgia Institute of Technology, Atlanta, GA, USA (e-mail: ye.zhao@me.gatech.edu)}
\thanks{$^{3}$School of Cybersecurity and Privacy, Georgia Institute of Technology, Atlanta, GA USA (e-mail: chou@gatech.edu).}
}

\maketitle

\thispagestyle{empty}

\begin{abstract}
We address the challenge of enabling bipedal robots to traverse rough terrain by developing probabilistically safe planning and control strategies that ensure dynamic feasibility and centroidal robustness under terrain uncertainty. Specifically, we propose a high-level Model Predictive Control (MPC) navigation framework for a bipedal robot with a specified confidence level of safety that (i) enables safe traversal toward a desired goal location across a terrain map with uncertain elevations, and (ii) formally incorporates uncertainty bounds into the centroidal dynamics of locomotion control. To model the rough terrain, we employ Gaussian Process (GP) regression to estimate elevation maps and leverage Conformal Prediction (CP) to construct calibrated confidence intervals that capture the true terrain elevation. Building on this, we formulate contraction-based reachable tubes that explicitly account for terrain uncertainty, ensuring state convergence and tube invariance. In addition, we introduce a contraction-based flywheel torque control law for the reduced-order Linear Inverted Pendulum Model (LIPM), which stabilizes the angular momentum about the center-of-mass (CoM). This formulation provides both probabilistic safety and goal reachability guarantees. For a given confidence level, we establish the forward invariance of the proposed torque control law by demonstrating exponential stabilization of the actual CoM phase-space trajectory and the desired trajectory prescribed by the high-level planner. Finally, we evaluate the effectiveness of our planning framework through physics-based simulations of the Digit bipedal robot in MuJoCo\footnote[1]{Videos of the simulated experiments in Mujoco can be found at \url{https://youtu.be/hoh2ilXRj4I?si=_qX5MEmHYbPuc1fx}}.
\end{abstract}
\section{Introduction}
Bipedal locomotion holds great promise for navigating unstructured and challenging environments, as they can adapt to irregular terrain through discrete and precisely controlled footstep placement~\cite{torres2022legged, gibson2022terrain}. However, locomotion over uncertain terrain remains susceptible to instability, particularly when consecutive footsteps must be placed on surfaces with significant geometric variability and centroidal momentum needs to be regulated accurately. Many biped navigation frameworks \cite{huang2023efficient,Muenprasitivej2024,shamsah2024terrainaware} address terrain uncertainty at the high-level planning but omit corrective motion strategies in low-level control with confidence-guaranteed on the uncertainty. To address this challenge, we develop a terrain-uncertainty-aware planning and control framework that formally quantifies terrain elevation uncertainty and leverages it to achieve provably-safe footstep planning with a specified probability threshold, while guaranteeing forward invariance of the centroidal states in a robust tube around the desired motion plan. Such a framework ensures locomotion balance and supports the generation of probabilistically-safe and dynamically-feasible locomotion plans over long horizons (see Fig.~\ref{fig:highlight}).

\begin{figure}[t]
    \centering
    \includegraphics[width=0.98\linewidth]{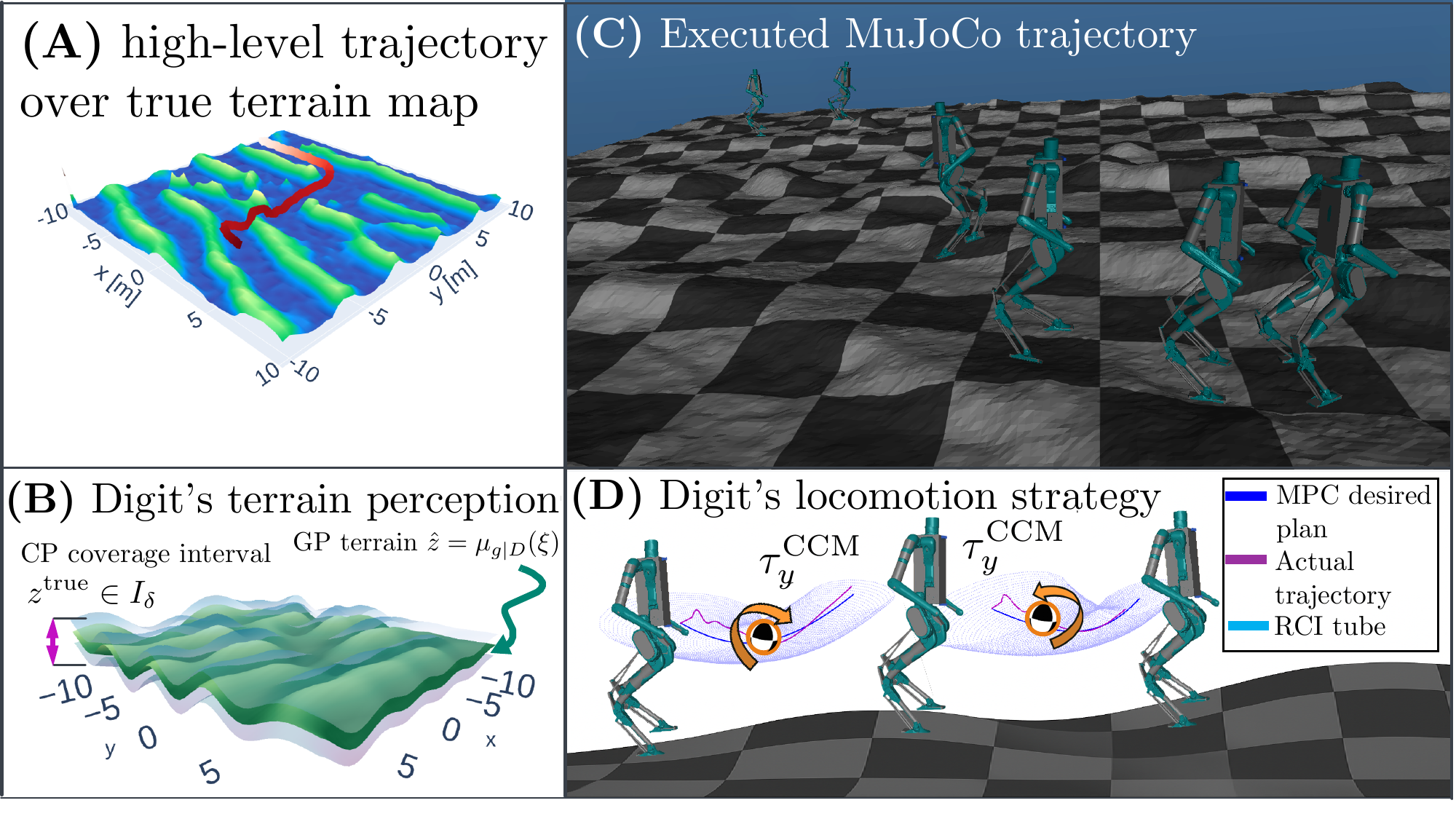}
    \caption{(A) A bird's-eye view of the robot navigation path over the true rough terrain map. (B) A visualization of terrain map estimated by Gaussian Process (GP) and conformal prediction (CP). (C) The bipedal robot Digit navigates through an environment
with rough terrain in our MuJoCo simulation. (D) Digit's flywheel torque control law maintains the CoM trajectory within the contraction-based robust control invariant tube on the robot with full-order dynamics.
}
    \label{fig:highlight}
\end{figure}

In our bipedal robot navigation framework, we leverage a reduced-order robot model, i.e., the Linear Inverted Pendulum Model (LIPM), to design a high-level planner for the center-of-mass (CoM) and footstep plans under a user-specified probability of safety. In particular, we derive uncertainty bounds on the LIPM-based robot dynamics caused by the terrain using conformal prediction (CP). Then, we use the derived uncertainty bounds to construct contraction-based reachable tubes around the nominal plans given by the high-level planner, while, at the low-level, a flywheel torque controller is designed to stabilize centroidal angular momentum induced by terrain variations. The integration of both uncertainty-aware CoM planning and the corrective torque controller guarantees dynamic safety under a user-defined probability, and the motion plans executed online are guaranteed to remain in a compact forward invariant set around the desired CoM motion plan. The key contributions and organization of this work are as follows:

\begin{itemize}
    \item We adopt a Gaussian process (GP) regression model with an Attentive Kernel (AK) for mapping nonstationary terrain. We then couple the GP model with CP to construct data-calibrated confidence intervals, ensuring high-confidence coverage of the true terrain elevations.

    \item We propose an uncertainty-informed model-predictive-control (MPC) framework for bipedal navigation over rough terrain. Given a predefined high probability of safety, the MPC guarantees provably-safe CoM planning toward a desired goal, up to a specified probability, via CP-informed footstep sequence selections and minimization of traversal terrain slopes. 
    
    \item To reliably track MPC plans, we derive a contraction-based flywheel torque control law for correcting off-nominal centroidal angular momentum caused by terrain perturbations. We use contraction theory and the CP bounds to compute time-varying reachable tubes, guaranteeing forward invariance and exponentially-stabilizing behavior toward the desired trajectory. 

    \item Finally, we discuss the gap of our contraction analysis relative to the full-order robot dynamics
    and demonstrate the validity of our planning and control strategy on a Mujoco simulation of the Digit robot.
\end{itemize}
\section{Related Work}
\subsection{Bipedal Locomotion over Rough and Uncertain Terrain} Numerous works have studied long-horizon bipedal planning over irregular, uncertain terrain \cite{gibson2022terrain, shamsah2024terrainaware, Muenprasitivej2024, huang2023efficient}. The work in~\cite{gibson2022terrain} presents a terrain-adaptive controller with piecewise linear terrain approximation and friction cone constraints to plan foot placements, but requires substantial prior terrain knowledge and ignores terrain uncertainty. As a closely related work to ours, the authors in~\cite{shamsah2024terrainaware} propose a terrain-aware MPC that penalizes the estimated slope to enable stable locomotion on steep inclines, but incorporates terrain uncertainty only in a traversability score for high-level navigation planning, not for footstep-level planning. The work of \cite{huang2023efficient} combines a Control-Lyapunov-Function-based controller with a sampling-based planner and traversability metrics to guide bipeds over undulating terrain while avoiding high-cost regions to preserve locomotion feasibility. Similarly, the work in \cite{Muenprasitivej2024} learns rough terrain online via a GP but conservatively bypasses high-elevation or uncertain regions for safety. Most of these approaches restrict robots to a limited set of navigation paths, whereas our work enables traversal of a broad range of uncertain regions while guaranteeing safe locomotion from a probabilistic perspective.

\subsection{Reduced-order  Locomotion Planning and Control}
In the regime of online optimal control for humanoid centroidal balancing, reduced-order models (ROMs) such as the LIPM \cite{kajita2001lip} have been widely used to maintain tractability and predict centroidal motion. With reduced computational burden, these models can be integrated with step-adaptation heuristics and centroidal momentum regulation to stabilize dynamic bipedal locomotion under model mismatch between the ROM-based planning and low-level control layers using full-body dynamics~\cite{dantec2024fromcentriod}. The work of \cite{zhao2017robust} introduces the Prismatic Inverted Pendulum Model (PIPM) for non-periodic CoM phase-space trajectories with 3D footstep planning on uneven terrain, using dynamic programming to compute flywheel torque for robustness to disturbances, while \cite{xiong2019orbitchar} employs the Hybrid-LIP (H-LIP) model with orbital energy characterization to design foot placements and integrates Backstepping-Barrier Functions in a QP framework to regulate leg forces and body height. The work in \cite{gibson2022terrain} uses the angular momentum LIP (ALIP) model for MPC-based footstep planning, where virtual constraints map ALIP outputs to joint commands, enabling accurate CoM and swing-foot tracking on uneven terrain. We leverage this ALIP model with flywheel torque control and use contraction analysis to derive a stabilizing controller that guarantees high-probability tracking of full-order centroidal dynamics. 

\subsection{Contraction Theory}
\looseness-1To address mismatches between ROM-based trajectories and the full-order dynamics, we use control contraction theory~\cite{lohmiller1998contraction} to design a tracking controller. Contraction-based controllers for control-affine systems~\cite{manchester2017CCM} provide incremental stabilizability and have been used to construct tracking tubes under disturbances in robotics~\cite{singh2023contract}. \cite{singh2023contract} designs uniform upper-bound ellipsoids around trajectories given disturbance bounds, while \cite{chou2022modelerrorprop} proposes spatially varying bounds that better reflect learned dynamics. Most work applies contraction-based tubes to high-level planning for obstacle avoidance, with limited focus on phase-portrait analysis of robot dynamics for low-level motion feasibility. To our knowledge, this is the first use of a contraction-based control for bipedal phase portraits, minimizing discrepancies between ROM references and full-order dynamics and enabling robust tracking under environmental and model uncertainties.

\subsection{Conformal Prediction}
To quantify uncertainty for robotic planning, \cite{lindemann2023CP} applies CP to quantify error in neural-network-predicted dynamic paths, guaranteeing true-path coverage at a predefined probability for collision avoidance. \cite{dawei2021uncertain-aware} derives disturbance bounds from the GP posterior for contraction analysis, but limited data can yield overly conservative bounds, and kernel assumptions may underestimate terrain error. In fact, \cite{pion2025GPwithCP} advocates using CP to refine GP prediction intervals. This work employs CP to tighten GP-derived bounds and calibrate them directly to data while preserving high-confidence safety for novel applications in locomotion.
\section{Preliminaries}\label{section: Preliminaries}

\subsection{Gaussian Processes}
\label{sec:GP definition}

To map terrain elevation with uncertainty for bipedal navigation, we employ GP regression:  

\begin{definition}[Gaussian Process Regression]
\label{def:GP}%
A GP models a function $g(\inp) \sim \mathcal{N}(\mu(\inp),\kappa(\inp,\inp))$ with mean $\mu:\mathbb{R}^n \to \mathbb{R}^n$ and covariance $\kappa:\mathbb{R}^n \times \mathbb{R}^n \to \mathbb{S}_{+}^n$.  
Given $m$ samples $D = \{(\inp\ith,\outp\ith)\}_{i=1}^m$, where $\inp\ith \in \mathbb{R}^n$ is the input (a terrain location $\inp\ith = [ x\ith, y\ith ]$ in the global frame) and $\outp\ith$ the elevation observed with Gaussian noise variance $\sigma_{\nu }^2$, let $K \in \mathbb{R}^{m\times m}$ be the covariance matrix with $K_{ij}=\kappa(\inp\ith,\inp\jth)$. For a test point $\inp'$, define $k(\inp')=[\kappa(\inp',\inp^1)\;\kappa(\inp',\inp^2)\;\ldots\;\kappa(\inp',\inp^m)]^T \in \mathbb{R}^m$. The predictive distribution of $g$ at $\inp'$ is Gaussian with mean and variance
\begin{subequations}
\label{eq:gp_posterior}
\begin{align}
     \mu_{g|D}(\inp') &= k(\inp')^T(K+\sigma_{\nu }^2I_m)^{-1}Z, \label{eq:gp_mean}\\
     \sigma_{g|D}^2(\inp') &= \kappa(\inp',\inp') - k(\inp')^T(K+\sigma_{\nu }^2I_m)^{-1}k(\inp'), \label{eq:gp_var}
\end{align}
\end{subequations}
where $I_m$ is the identity and $Z = [z^1,\ldots,z^m]^T$.
\end{definition}

\subsubsection{Attentive Kernel} 
\looseness-1We use the Attentive Kernel (AK)~\cite{chen2022ak} for terrain elevation estimation. The AK adapts to spatial variability through a neural network combining multiple radial basis function (RBF) kernels, with a secondary network assigning membership vectors to decouple nearby correlations. In contrast, the stationary RBF kernel yields uniformly smooth predictions $\kappa^{\rm RBF}(\inp\ith,\inp\jth) = \sigma_{f}^2 \exp(-\|\inp\ith - \inp\jth\|^2 / 2\ell^2)$, with variance $\sigma_{f}^2$ and length-scale $\ell$. 
The AK is defined as
\begin{equation}
\nonumber
\kappa^{\rm AK}(\inp\ith,\inp\jth) = \alpha \bar{z}^T \bar{z}' + \textstyle\sum_{m=1}^{M} \bar{w}_m \kappa_m^{\rm RBF}(\inp\ith,\inp\jth)\bar{w}_m',
\end{equation}
where $\alpha$ is a coefficient, $\bar{w}$ and $\bar{z}$ are learned weight and membership vectors, and $\{\kappa_m^{\rm RBF}(\inp\ith,\inp\jth)\}_{m=1}^M$ are RBF kernels with length-scales $\{\ell_m\}_{m=1}^M$.

\subsection{Robot Model}\label{subsection: Robot Model}
\subsubsection{Reduced-order robot dynamics}

\begin{figure}[t]
    \centering
    \subfigure[High-level LIPM plan (discrete)]{\includegraphics[width=0.48\linewidth]{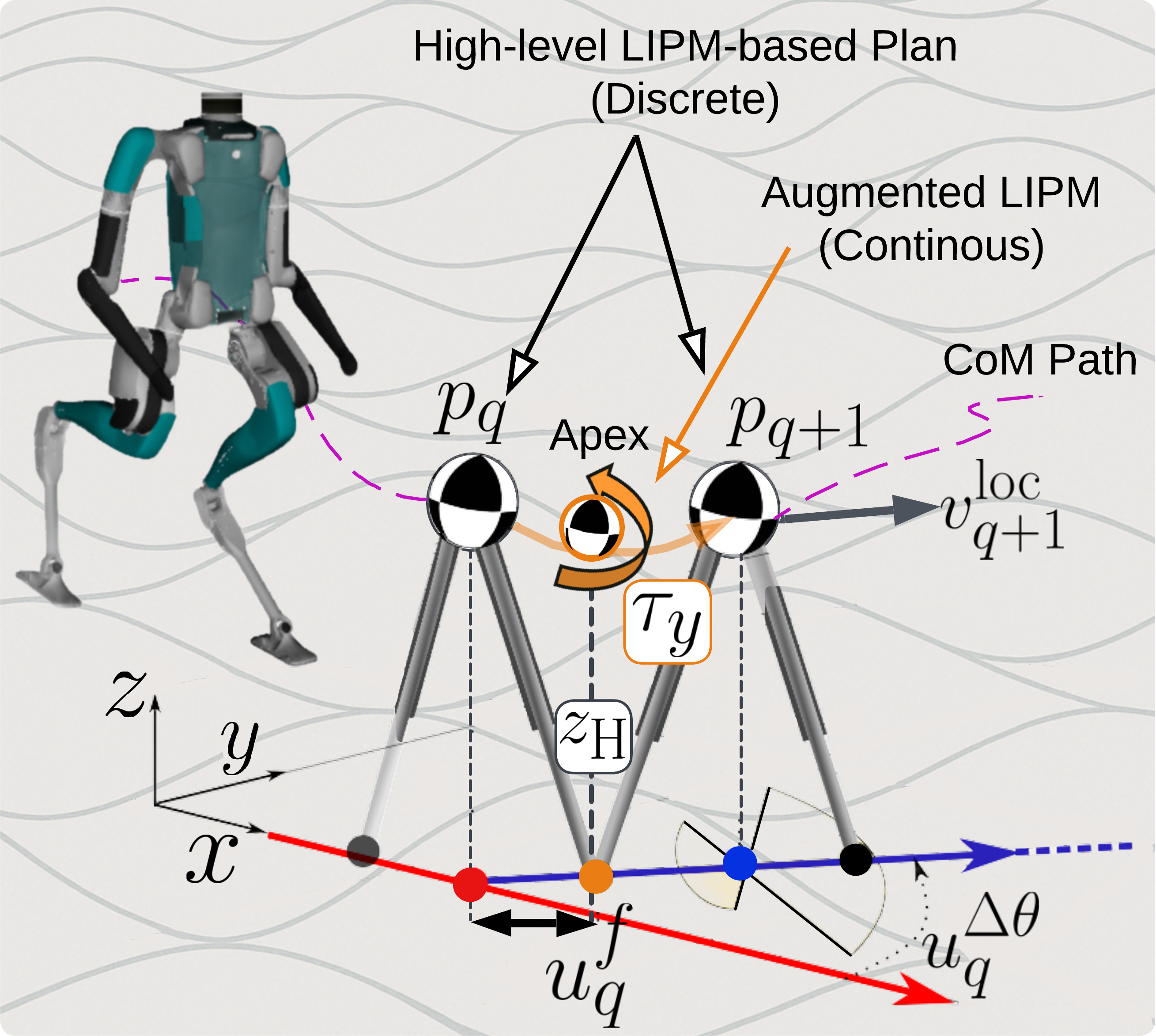}}
    \subfigure[Augmented LIPM (continuous)]{\includegraphics[width=0.5\linewidth]{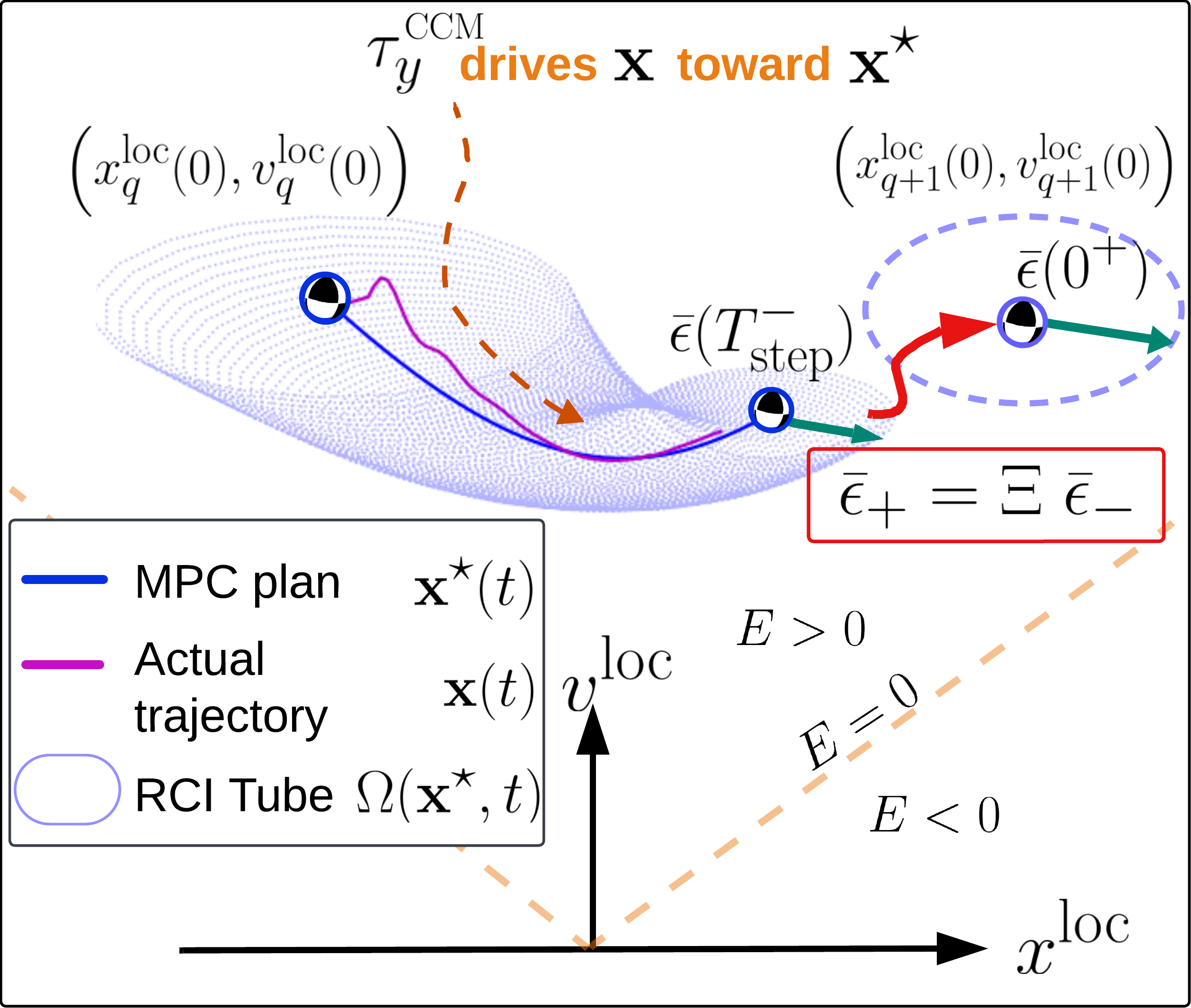}}
\caption{(a) The high-level planner's global dynamics is based on the Linear Inverted Pendulum Model (LIPM). For contraction analysis at the low-level, we use the Augmented LIP Model (Aug-LIPM) with flywheel torque $\tau_y$ about the CoM.  
(b) Sagittal phase portrait for one walking step from $q^{\rm th}$ to $(q+1)^{\rm th}$ in the local frame of the $q^{\rm th}$ footstep. The desired MPC-guided trajectory $\state^\star(t)$ (dark blue) is designed by \eqref{eq:sagittal_cons} given the MPC output $(x\local\currq,v\local\currq,u^f\currq)$. The true CoM trajectory from full-order dynamics $\state(t)$ (pink) is regulated to track the desired plan via CCM control law $\tau_y^{\rm CCM}$. The RCI tube $\Omega(\state^\star ,t)$ (light blue) is constructed around the desired trajectory; orbital energy $E$ is represented by the asymptotic slope line (orange), and tube-bound propagation is shown via the saltation matrix $\Xi$.
}
    \label{fig:LIPM_dynamics}
\end{figure}

\looseness-1We first introduce the continuous LIP dynamics in the sagittal direction and in the local frame of the robot's stance foot, which is governed by $\ddot{x}\local=\omega^2x\local$,
where $x\local \in \mathbb{R}$ is the CoM sagittal position in the local frame of current stance foot,  $\omega = \sqrt{g/z_{\rm H}}$ is the asymptotic slope, with the gravitational acceleration $g$  and the CoM height at the apex state $z_{\rm H}$\footnote{Apex state is the state when the CoM is directly on top of the foot.}, as seen in Fig. \ref{fig:LIPM_dynamics}(a), which is constant for all walking steps (i.e., the CoM surface plane is compliant with the terrain shape \cite{zhao2017robust}).
Given the local frame of the stance foot, where the distance to the CoM position is denoted by $-u^f$ in the sagittal direction, the closed-form evolution of the sagittal dynamics can be expressed as  
\begin{subequations}\label{eq:sagittal_cons}
    \begin{align}
 x^{\rm loc}(t) &= x^{\rm loc}(0) + \frac{\sinh(\omega t)}{\omega}\,v^{\rm loc}(0) + \big(1-\cosh(\omega t)\big)u^f, \label{eq:sagittal_x_cons}\\ 
 v^{\rm loc}(t) &= \cosh(\omega t)\, v^{\rm loc}(0) - \omega\sinh(\omega t)\,u^f.\label{eq:sagittal_velocity_cons}
\end{align}
\end{subequations}

For high-level locomotion planning, we time-discretize the LIP model with each discrete state defined at the foot-switching instant. By fixing constant the step duration $T_{\rm step}$ in \eqref{eq:sagittal_cons}, the state at the $(q+1)^{\rm th}$ step is $x\local_{q+1} \doteq x\local_q(T_{\rm step})$, where $x\local_q(T_{\rm step})$ is the state of the continuous dynamics \eqref{eq:sagittal_x_cons} evaluated at time instant $T_{\rm step}$.
leading to the discrete sagittal LIP local dynamics: 
\begin{subequations}
\begin{align}
 x^{\rm loc}_{q+1} &= x^{\rm loc}_{q} + \frac{\sinh(\omega T_{\rm step})}{\omega}\,v^{\rm loc}_{q} + \big(1-\cosh(\omega T_{\rm step})\big)u^f_q, \label{eq:sagittal_x}\\
 v^{\rm loc}_{q+1} &= \cosh(\omega T_{\rm step})\, v^{\rm loc}_{q} - \omega\sinh(\omega T_{\rm step})\,u^f_q.\label{eq:sagittal_velocity}
\end{align}
\end{subequations}

 Following the formulation in~\cite{shamsah2024social}, the bipedal global state at timestep $q$ is defined as $\state_q = (p_q, v_q\local, \theta_q)\in \mathbb{R}^5$, where $p_q = (x_q,y_q,z_q)$ is the CoM position in the global frame, $v_q\local$ is the local sagittal velocity, and $\theta_q$ is the global heading angle. The control variables are  $\ctrl_q = (u_q^f, u_q^{\Delta \theta})\in\mathbb{R}^2$, where $u_q^f$ is the sagittal foot position relative to the CoM, and $u_q^{\Delta \theta}$ is the global heading change between two consecutive steps, as shown in Fig. \ref{fig:LIPM_dynamics}(a).

Applying a coordinate transformation with respect to $\theta_q$ to \eqref{eq:sagittal_x}, \eqref{eq:sagittal_x} yields the 3-D LIP dynamics in the global coordinates (derivations in~\cite{shamsah2024social}):  
\begin{subequations}
\label{eq:lip_dynamics}
\begin{align}
 x\nextq &= x_q + \Delta x\local \cos(\theta_q), \\
 y\nextq &= y_q + \Delta x\local \sin(\theta_q), \\
 z\nextq &= z_q + \nabla_{x,y} \mu_{g\mid D}(x_{q}, y_{q})\Delta x^{\rm loc} , \\
 v\local\nextq &= \cosh(\omega T) v\local\currq - \omega \sinh(\omega T) u^f\currq, \\
 \theta\nextq &= \theta\currq + u^{\Delta \theta}\currq,
\end{align}
\end{subequations}
where $\Delta \sagx\local = \sagx\local\nextq - \sagx\local\currq$ is derived from~\eqref{eq:sagittal_x}. The heading update is given by $u^{\Delta \theta}\currq = \theta\nextq - \theta\currq$ and $\nabla_{x,y} \ \mu_{g|D}(x\currq,y\currq)$ is the estimated slope of the GP terrain map \cite{shamsah2024terrainaware}. For compactness, system~\eqref{eq:lip_dynamics} is denoted as $\state \nextq = \Phi(\state \currq, \ctrl \currq)$.

\medskip
\subsubsection{Augmented LIP dynamics}\label{sec:aug_lipm}
For contraction analysis, we employ an Augmented LIP Model (Aug-LIPM) with the addition of flywheel torque control about the local pitch angle $\tau_y$ and bounded disturbance input $\disturb$, which is governed by 
\begin{equation}\label{eq:AugLIP_doubleint}
    \ddot{x}\local=\omega^2x\local- \textstyle\frac{\omega^2}{mg}\tau_y + \disturb,
\end{equation}
for a bipedal robot with mass $m$. We assume $\disturb\in \mathcal{W} \subseteq \mathbb{R}$, where $\mathcal{W}$ is a compact set. We provide the state-space form of the Aug-LIP for the ease of notation interchange in contraction analysis (Sec. \ref{sec:CCM}) as:
\begin{align}\label{eq:AugLIP_dyn_CCM}
\underbrace{\begin{bmatrix}
\dot{x}\local \\
\dot{v}\local
\end{bmatrix}}_{\dot{\state}\local}
&=
\underbrace{\begin{bmatrix}
0 & 1 \\
\omega^2 & 0
\end{bmatrix}}_{A}
\underbrace{\begin{bmatrix}
x\local \\
v\local
\end{bmatrix}}_{\state\local}
+
\underbrace{\begin{bmatrix}
0 \\
-\frac{\omega^2}{mg}
\end{bmatrix}}_{B}
\underbrace{\tau_y}_{\ctrl\local}
+
\underbrace{\begin{bmatrix}
0 \\
1
\end{bmatrix}}_{B_{\disturb}}\disturb.
\end{align}

\subsection{Conformal Prediction}\label{sec:CP}
Consider a collection of $k+1$ random variables $R^{(0)},\hdots,R^{(k)}$ that are \textit{exchangeable}\footnote{Exchangeability means that the joint distribution of $R^{(0)},\hdots,R^{(k)}$ is invariant under any permutation $\sigma$ of the indices $\{0,\hdots, k\}$. It is a weaker assumption than independence and identical distribution (i.i.d.).}. These variables, referred to as \emph{nonconformity scores}, quantify the deviation between predictions and observations. In supervised learning, a common score is   $R\ith = R(\inp\ith,\outp\ith|D)=\big|\outp\ith - \mu_D(\inp\ith)\big|$,
where $\mu_D(\cdot)$ is a predictive model trained on dataset $D$ to estimate the observation $\outp\ith$ from input $\inp\ith$. Larger scores correspond to poorer predictive accuracy. The aim of CP is to bound the nonconformity score of a query point $R^{(0)} = |\outp'_{\rm true}-\mu_D(\inp')|$ with high probability. For failure rate $\delta \in (0,1)$, we seek a threshold $\mathcal{C}$ such that $P(R^{(0)} \leq \mathcal{C}) \geq 1-\delta$.  

A quantile-based approach \cite[Lemma 1]{tibshirani2019conformal} constructs $\mathcal{C}$ as the $(1-\delta)$-quantile of the empirical distribution of $\{ R^{(1)},\hdots,R^{(k)}\}\cup\{\infty\}$ sorted in non-decreasing order. With this approach, we can define the threshold $\mathcal{C} = R^{(p)}$ where $p = \lceil (k+1)(1-\delta) \rceil$, where $\lceil \cdot \rceil$ is the ceiling function. 

\begin{definition}[Split Conformal Prediction]
\label{def:splitCP}
Split CP provides computational efficiency by partitioning the dataset $D = D^{\rm train} \cup D^{\rm cal}$ into disjoint training and calibration sets. The predictor $\mu(\cdot)$ is fitted on $D^{\rm train}$, and nonconformity scores $R\ith$ are computed on $(\inp\ith, \outp\ith) \in D^{\rm cal}$ for $i=1,\hdots,k$ with $k = |D^{\rm cal}|$. For a query point $\inp'$ with unknown truth value $\outp^{\rm true}$, the prediction interval $I_\delta$ is given by:
\begin{equation}\label{eq:confidence_interval}
     I_\delta(\inp') = \big\{ \outp \in \mathbb{R} : R(\inp', \outp \ | D^{\rm train}) \leq \mathcal{C} \big\},
\end{equation}
where $\mathcal{C} = R^{(p)}$ and failure rate $\delta$ is user-defined. Under exchangeability of the data, the coverage interval \eqref{eq:confidence_interval} ensures $
\mathbb{P}\big(\outp^{\rm true} \in I_\delta(\inp')\big) \geq 1 - \delta$, or equivalently,
\begin{align}\label{eq:coverage_prob}
 \mathbb{P}\big(\outp^{\rm true} \in [\mu_{D^{\rm train}}(\inp') - \mathcal{C} , \mu_{D^{\rm train}}(\inp') + \mathcal{C}]\big) \geq 1 - \delta.
\end{align}
\end{definition}

\subsection{Control Contraction Metrics}\label{sec:CCM}
Contraction theory studies incremental stability by examining differential dynamics between neighboring trajectories. In this work, we use control contraction metrics (CCMs) to stabilize our reduced-order model \cite{manchester2014universalstab}. 
For simplicity, all discussion in this section is presented in the local frame of the current foot stance, and the superscript $\local$ is omitted. Given a linear time-invariant system of form \eqref{eq:AugLIP_dyn_CCM} with state and control $\state \in \mathbb{R}^{n_x}$, $\ctrl \in \mathbb{R}^{n_u}$ under disturbance $\mathbf{\disturb} \in \mathbb{R}^{n_w}$, a CCM defines a state-feedback tracking control law that ensures that any \textit{disturbed} system trajectory $\state(t)$ converges toward a desired \textit{disturbance-free} reference trajectory $\{\state^\star(t),\ctrl^\star(t)\}_{t \geq 0}$ exponentially quickly, i.e.,
for any $\state(0)$, $\lim_{t\to\infty} \|\state(t) - \state^\star(t)\|_2 = 0$, there exists constants $\Lambda, \lambda > 0$ such that $\|\state(t) - \state^\star(t)\|_2 \leq \Lambda e^{-\lambda t} \|\state(0) - \state^\star(0)\|_2 $. The decay rate $\lambda$ is called the \textit{contraction rate}. Specifically, for LTI systems, a CCM $ M\in \mathbb{S}_+^{n_x}$ satisfies the following condition
\begin{equation}\label{eq:strong} 
\small A^\top M + M A - M B B^\top M \preceq -2\lambda M.
\end{equation}
The resulting feedback law $\ctrl(t) = -\frac{1}{2}\rho B^\top M(\state(t)-\state^*(t))$ renders the noise-free, closed-loop system incrementally exponentially stable, converging to $\state^*(t)$ with contraction rate $\lambda > 0$. While \eqref{eq:strong} is non-convex, it can be reformulated via variable transformation into a convex semidefinite program (SDP), which can be efficiently solved \cite{manchester2017CCM}.

\subsubsection{CCMs, Riemannian Energy, and Invariant Tubes}
In the presence of bounded disturbance $\mathcal{W} = \{\disturb \mid \|\disturb\|_2\leq\bar{\disturb}\}$, instead of ensuring exponential convergence to the reference, the CCM-based controller can ensure that the closed-loop system remains within a robust control invariant tube. To set the stage, we define the Riemannian energy $\mathcal{E}(\state, \state^\star)=(\state(t) - \state^\star(t))^\top M(\state(t) - \state^\star(t))$. From \cite{singh2023contract}, we have
\begin{equation}
    \small\mathcal{E}(\state^\star(t), \state(t)) \leq \left[\sqrt{\mathcal{E}(\state^*(0), \state(0))}\, e^{-\lambda t} + \bar{d}(1 - e^{-\lambda t})\right]^2.\label{eq:riem_energy_bound}
\end{equation}
where $\bar{d} = \sigma(M^{\frac{1}{2}}B_{\disturb})\bar{\disturb}/\lambda$, where $\sigma$ is the singular value. 

\begin{definition}[Robust Control Invariant Tube]\label{def:RCItube}
    A robust control invariant (RCI) tube \cite{chou2022modelerrorprop,DBLP:conf/wafr/ChouOB22,singh2023contract} 
    \begin{equation}\label{eq:rci_tube}
        \Omega(\state^\star, t) := \{ \state\in\mathcal{X}: \| \state(t) - \state^\star(t)\|_2 \leq \bar{\epsilon}(t) , \forall t\in[0,T_{\rm step}]\}
    \end{equation} is defined such that any state starting inside $\Omega(\state^\star, t)$ is guaranteed to remain inside it throughout one walking step of duration $T_{\rm step}$. Assuming the system \eqref{eq:AugLIP_dyn_CCM} with $ \|\disturb\|_2 \leq\bar{\disturb}$, we construct a time-varying ellipsoid centered around the reference state $\state^\star(t)$. The maximum radial distance $\bar\epsilon(t)$ under $\bar{\disturb}$ which preserves contracting behavior is given by:
    \begin{equation*}
    \textstyle\small\bar{\epsilon}(t) = \sqrt{ \left( \int_{0}^{t} \bar{\mathcal{E}}(t)\, dt \right) \frac{1}{\sqrt{\underline{\lambda}(M)}}},
    \end{equation*}
    where $\bar{\mathcal{E}}(t)$ is the right-hand side of  \eqref{eq:riem_energy_bound}, and $\underline{\lambda}(M)$ is the minimum eigenvalue of $M$.
\end{definition}

\subsubsection{Saltation Matrix}\label{def:saltation_matrix} The saltation matrix $\Xi$ is a first-order approximation of a
system's sensitivity to discrete events \cite{kong2024saltation}, used to correctly propagate the variation about the nominal trajectory $\delta x(t)$ when a hybrid jump occurs such that $\delta x(t^+)=\Xi\delta x(t^-)$, where $\delta x(t^-), \delta x(t^+)$ denotes the pre-transition and post-transition, respectively. The dynamics model of a bipedal robot is hybrid, consisting of continuous single-contact dynamics described by \eqref{eq:AugLIP_dyn_CCM} (i.e., walking on one stance leg) and discrete double-contact transition (i.e, at the leg switching instance) with assumed duration $T_{\rm switch}$, for which we define the \textit{guard condition} $g$ and \textit{reset map} $\Delta:\state_-\to\state_+$ as:
\begin{align}\nonumber
     g&:x\local(0^+)=\Tilde{x}\local(T_{\rm step}^-), \\
 \nonumber\Delta&:\begin{cases}
  x\local(0^+) &= x\local(T_{\rm step}^-) + v\local(T_{\rm step}^-)\cdot T_{\rm switch}-u^f, \\
  v\local(0^+) &= v\local(T_{\rm step}^-), 
    \end{cases}
\end{align}

\noindent where $\Tilde{x}\local(T_{\rm step}^-)$ is predicted using \eqref{eq:sagittal_x_cons} given $(x\local(0^-),v\local(0^-), u^f)$.
Since the RCI tube in Def. \ref{def:RCItube} relies on continuous propagation of \eqref{eq:riem_energy_bound}, we adopt the saltation matrix $\Xi$ to formally define how the tube bound changes over the discrete transition, which is defined as
\begin{align}
    \nonumber \Xi = J_{\Delta} + \frac{(\mathcal{F}^+ - J_{\Delta}\mathcal{F}^-) J_g^{\top} }{J_g^{\top} \mathcal{F}^-}, \label{eq: tradsalt}
\end{align}

\noindent where $J_{\Delta}$ and $J_{g}$ are the Jacobian matrix of $\Delta$ and $g$ evaluated at $\state_{-}$, respectively, and $\mathcal{F}^{\pm}=[v\local(t) ,\omega^2x\local(t)]^\top$ represents the vector field evaluated at state $\state_{\pm}$.
\section{Methods}

\subsection{Framework Overview}
We propose a framework for bipedal navigation over rough terrain with elevation uncertainty. From sparse elevation data, the terrain is estimated using a GP model $\mu_{g|D}(\inp)$ and combined with CP to obtain coverage intervals $I_\delta$ of the true height map (Sec. \ref{sec:method_UQ}). A high-level MPC planner then generates CoM trajectories and footstep sequences with probabilistic safety guarantees (Sec.~\ref{sec:MPC}). CP-based uncertainty quantification defines contraction-based reachable tubes $\Omega(\state^{\rm loc,\star},t)$ around the MPC plans (Sec. \ref{sec:RCI_tube}), while a contraction-based flywheel torque controller stabilizes off-nominal centroidal angular momentum from terrain uncertainty (Sec.~\ref{sec:CCM_for_fullorder}). The overall method is shown in Fig.~\ref{fig:block_diagram}.

\begin{figure}[t]
    \centering
    \includegraphics[width=0.93\linewidth]{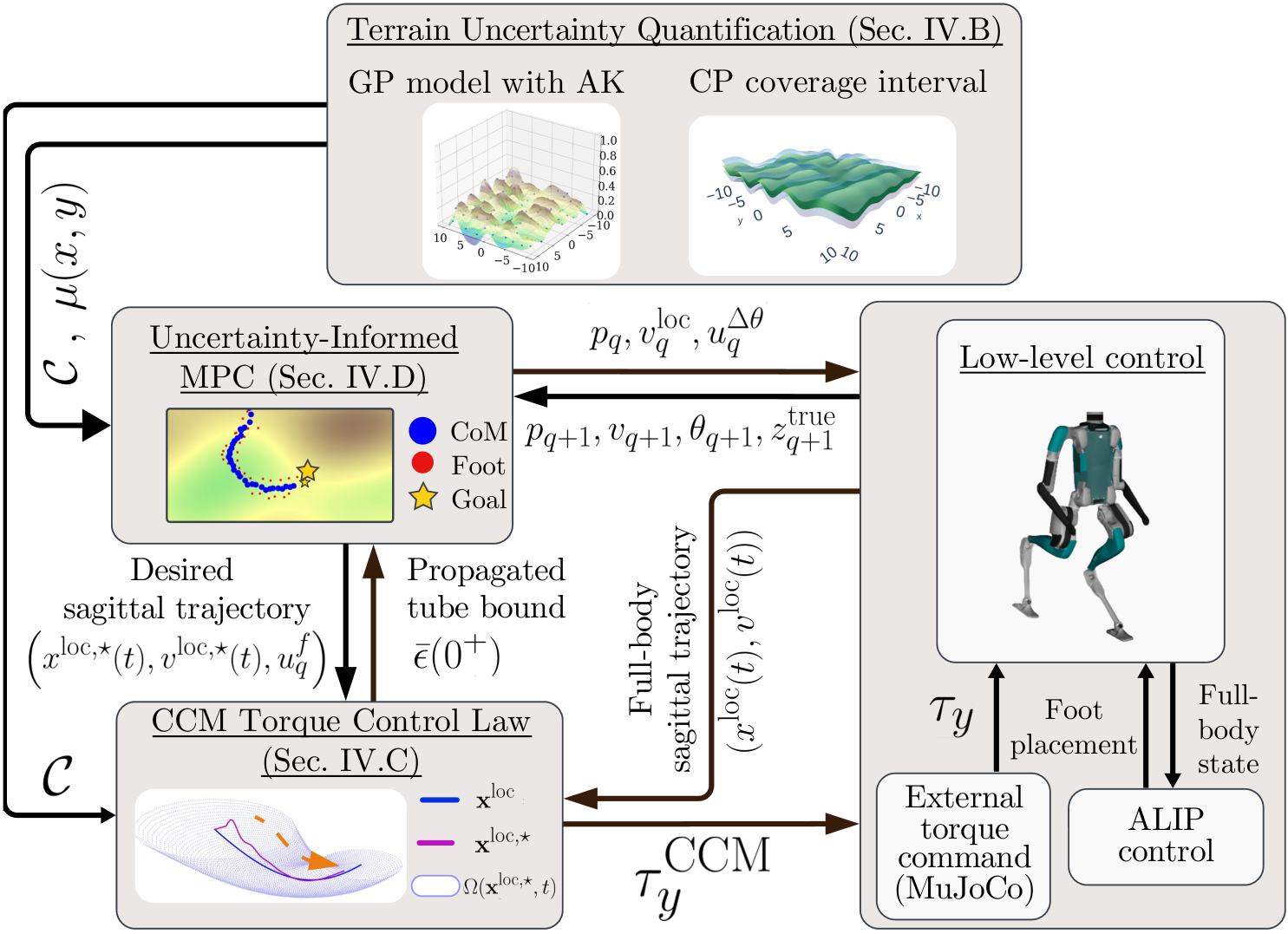}
    \caption{Overall block diagram of the proposed probabilistically-safe planning and control strategy for bipedal navigation over uncertain terrain.}
    \label{fig:block_diagram}
\end{figure}

\subsection{Uncertainty Quantification for Bipedal Footstep Planning}\label{sec:method_UQ}

\subsubsection{Split CP with GP model}
We apply Split CP to the GP mean terrain estimate to construct coverage intervals with predefined probability. Unlike GP confidence intervals, which become overly conservative in sparse data due to large variance, Split CP uses the empirical error distribution $R\ith$ to determine a threshold $\mathcal{C}$, yielding tighter intervals of the form $I_\delta=[\mu_{g|D^{\rm train}}(x,y)-\mathcal{C},\,\mu_{g|D^{\rm train}}(x,y)+\mathcal{C}]$, as validated in Sec.~\ref{sec:results}.
For brevity, we denote $\mu(\cdot)$ as the GP terrain model trained on the split dataset for the remainder of this paper.

\subsubsection{CP safe footstep constraint}
We derive the safety constraint for terrain elevation change between two adjacent footsteps defined by $(x_{q}, y_{q}, z_{q})$ and $( x_{q+1}, y_{q+1} , z_{q+1})$. We define an inequality $c(z_{q}, z_{q+1}) \geq 0$ to bound the height change by the maximum feasible limit $\Delta h_{\rm max}$, where
\begin{equation}\label{eq:c_hmax}
     c(z_{q}, z_{q+1}) = \Delta h_{\rm max} - | z_{q+1} - z_{q}|.
\end{equation}
\looseness-1At each iteration of high-level planning, we assume $z\currq = z\currq^{\rm true}$ to be known (i.e., the current foot height can be measured) and assume mean height, $z_{q+i}' = \mu(x\nextq, y\nextq)$, for all future footsteps for $i\in\{1,\ldots,H-1\}$ over the horizon $H$. We now prove that enforcing \eqref{eq:c_hmax} ensures a feasible step with probability at least $1-\delta$ despite terrain uncertainty.

\begin{lemma}[CP Safe Footstep Constraint]\label{lemma:CP_safe}
Let $c:\mathbb{R}^n \times \mathbb{R}^n \to \mathbb{R}$ be a Lipschitz continuous function in its second argument with Lipschitz constant $L>0$, i.e.,
$|c(\cdot,a)-c(\cdot,b)| \leq L\|a-b\|$ for all $ a,b$.
Let $\mathcal{C}$ denote the $(1-\delta)$-quantile of the empirical distribution of random variables $R^{(1)},\dots,R^{(k)}$.  
If the constraint
\begin{equation}
 c(z_{q}, z_{q+1}') \,\geq\, L\mathcal{C}
\end{equation}
is enforced at step $q$, then the true terrain height change between $z_{q}^{\rm true}$ and $z_{q+1}^{\rm true}$ remains below the feasible limit $\Delta h_{\rm max}$ with probability at least $1-\delta$.  
\end{lemma}

\begin{proof}
At the $q^{\text{th}}$ step we have $c(z_{q}^{\rm true}, z_{q+1}') \,\geq\, L\mathcal{C}$, 
since at the first timestep, $z_q = z_q^{\textrm{true}}$. Since $c(\cdot,\cdot)$ is $L$-Lipschitz in its second argument, we have  
\begin{equation}
 c(z_{q}^{\rm true}, z_{q+1}^{\rm true})
\;\geq\; c(z_{q}^{\rm true}, z_{q+1}') 
- L\big\|z_{q+1}^{\rm true}-z_{q+1}'\big\|.
\label{eq:CPsoft1}
\end{equation}
Substituting the constraint into \eqref{eq:CPsoft1} yields  
\begin{equation}
  c(z_{q}^{\rm true}, z_{q+1}^{\rm true})
\;\geq\; L(\mathcal{C} - \big\|z_{q+1}^{\rm true}-z_{q+1}'\big\|).
\label{eq:CPsoft2}
\end{equation}
Define the random variable 
\(R = \big\|z_{q+1}^{\rm true}-z_{q+1}'\big\|\).  
By the definition of $\mathcal{C}$, 
$\mathbb{P}\!\left(\mathcal{C}-R\geq 0\right) \,\geq\, 1-\delta$.
Thus, the right-hand side of \eqref{eq:CPsoft2} is nonnegative with probability at least $1-\delta$. Consequently,  
$\mathbb{P}\!\Bigl(c(z_{q}^{\rm true}, z_{q+1}^{\rm true}) \,\geq\, 0\Bigr) \,\geq\, 1-\delta$, 
which guarantees that the true terrain elevation change at step $q$ satisfies the feasible limit with confidence level $1-\delta$.  
\end{proof}

\subsection{RCI Tube for Phase Space Planning}\label{sec:RCIforBiped}
\subsubsection{Disturbance bounds for Augmented LIP Model}\label{sec:PIPMdisturbance}

Given an estimated height $\mu(x\currq,y\currq)$ at the current stance foot of the $q^{\rm th}$ step and a threshold $\mathcal{C}$, the true terrain height at the $q^{\rm th}$ step, denoted as $\outp^{\rm true}\currq$, lies in the interval $\outp^{\rm true}\currq \in I_\delta^q$ with a confidence level of $(1-\delta)$. The Aug-LIPM dynamics \eqref{eq:AugLIP_dyn_CCM} are propagated under the assumption that $\omega^2 = \frac{g}{z_{H}}$, which remains constant due to the assumption of a fixed apex height $z_H$. This corresponds to constraining the CoM motion to a surface plane parallel to the terrain profile \cite{zhao2017robust}. We account for uncertainty in $\outp^{\rm true}_q$ by propagating it to the asymptotic slope, which arises when the true terrain height deviates from the estimated height used in the Aug-LIPM. Specifically, instead of assuming a fixed $z_H$, we have $\omega_{\rm true}^2\in[\frac{g}{(z_{H} + \mathcal{C})}, \frac{g}{(z_{H} - \mathcal{C})}]\doteq\omega^2+C_\Delta$, where $C_{\Delta} \in [\frac{-g\mathcal{C}}{z_{\rm H}(z_{\rm H }+\mathcal{C})},\frac{g\mathcal{C}}{z_{\rm H}(z_{\rm H}-\mathcal{C})}]$.
Thus, the Aug-LIPM in \eqref{eq:AugLIP_doubleint} becomes
\begin{equation}\label{eq:auglip_terrain}
    \ddot{x}\local=\omega^2x\local- \frac{\omega^2}{mg}\tau_y + \disturb_{\rm terrain},
\end{equation}
 where $\disturb_{\rm terrain} = C_{\Delta}(x\local - \frac{\tau_y}{mg})$, and $x\local\in \mathcal{X}\local \doteq [x\local_{\min},x\local_{\max}]$ and $\tau_y\in \mathcal{U}\local\doteq [\tau_{y,\min},\tau_{y,\max}]$ are bounded by feasible state and control constraints.

We then define a terrain-uncertainty related bounded set of disturbances $\disturb$ in \eqref{eq:AugLIP_dyn_CCM} as $\mathcal{W}_\textrm{terrain} \doteq \{\disturb : \Vert \disturb\Vert \le \bar{\disturb}_{\rm terrain}\}$, where we use the upper bound $\bar{\disturb}_{\rm terrain}= \max_{\{x\in\mathcal{X}\local, \tau\in\mathcal{U}\local\}}|C_{\Delta}| |x-\frac{\tau}{mg}|.$

\subsubsection{RCI tube for sagittal phase-space dynamics}\label{sec:RCI_tube}
We define the reference trajectory for the Aug-LIPM $\{\state^{\rm loc, \star}(t)\, \ctrl^{\rm loc, \star}(t)\}$ as sagittal trajectory guided by 
    \eqref{eq:sagittal_cons} given high-level planner's output $(x\local\currq,v\local\currq,u^f\currq)$ (i.e., $\ctrl^{\rm loc, \star}=\tau_y^\star=0$. Given the sytem \eqref{eq:AugLIP_dyn_CCM} is a linear time-invariant with control matrix $B$, the CCM $M$ in \eqref{eq:strong} provides the flywheel torque control law that exponentially stabilizes the true trajectory toward $\state^{\rm loc, \star}(t)$
\begin{equation}\label{eq:ccm_torque}
    \tau_y^{\rm CCM}(t) = - \frac{1}{2}\rho B^\top M(\state\local(t) - \state^{\rm loc, \star}(t)).
\end{equation}

\begin{remark}
     CoM sagittal states $\state\local$ governed by \eqref{eq:AugLIP_dyn_CCM} that are initialized in the RCI tube \eqref{eq:rci_tube} are guaranteed to stay within the tube for all $t\in[0,T_{\rm step}]$ under the CCM control law $\tau_y^{\rm CCM}$, which ensure forward invariance around the desired motion plans $\{\state^{\rm loc, \star},\ctrl^{\rm loc,\star}\}$. We also emphasize that exiting the RCI tube does not imply a fall; it only implies that forward invariance in \eqref{eq:rci_tube} is not guaranteed; in practice, \eqref{eq:ccm_torque} can still effectively compensate for terrain-uncertainty-induced perturbations $w_{\rm terrain}$.
\end{remark}

\subsubsection{Saltation matrix tube propagation}
\looseness-1During the discrete transition, corresponding to the foot-switching instant between walking steps, we propose a formal propagation rule for the upper bound of the Riemannian energy across the transition using the saltation matrix. The propagation is given by:
$
\bar{\mathcal{E}}^{+}(0) = \bm{\epsilon}^{\top} \Xi^{\top} M \Xi \bm{\epsilon}$, 
where $\bm{\epsilon} \in \mathbb{R}^2$ is any vector of norm $\Vert \bm{\epsilon}\Vert_2 = \bar \epsilon(T_{\rm step}^-)$, which is the RCI tube bound at the end of the step. By the construction of our guard condition and reset map in Def. \ref{def:saltation_matrix}, the determinant $\| \Xi\|>1$ implies that $\Xi$ is expansive. Given this result, we prove that the sagittal state of the robot is guaranteed to remain forward invariant over a single step with probability at least $1-\delta$ (Lem. \ref{lemma:safe_prob_of_RCI}) and over $i$ steps with probability at least $(1-\delta)^i$ (Lem. \ref{lemma:RCI_multiple_steps}).

\begin{lemma}\label{lemma:safe_prob_of_RCI} 
Given terrain coverage interval with confidence level $1-\delta$, the sagittal state governed by \eqref{eq:AugLIP_dyn_CCM} initialized inside the RCI tube $\Omega(\state^{\rm loc, \star},t)$ is guaranteed to stay within the tube under the CCM control law $\tau_y^{\rm CCM}(t)$ \eqref{eq:ccm_torque} for all $t\in[0,T_{\rm step}]$ during \textit{one} walking step under any disturbance in terrain height with probability $1-\delta$.
\end{lemma}
\begin{proof}
Following \eqref{def:RCItube}, the RCI tube $\Omega(\state^{\rm loc, \star},t)$ satisfies an upper-bound $\bar{\epsilon}(t)$ determined by $\bar{\disturb}_{\rm terrain}$ such that $ \| \state\local(t) - \state^{\rm loc,\star}(t)\| \leq \bar{\epsilon}(t) , \forall t\in[0,T_{\rm step}]$. From the CP coverage guarantees, we have that $w \in \mathcal{W}_\textrm{terrain}$, i.e., the disturbance bound is valid, 
 with probability at least $1-\delta$. Thus, $\state^\textrm{loc}(t) \in \Omega(\state^{\rm loc, \star},t)$ with probability at least $1-\delta$.
\end{proof}

\begin{lemma}\label{lemma:RCI_multiple_steps}
\looseness-1Over $i$ walking steps, the probability of the terrain-perturbed dynamics \eqref{eq:auglip_terrain} staying in the RCI tube for $i \in \{1,\ldots,H-1\}$ is at least $(1-\delta)^i$, where $H$ is the horizon.
    
\end{lemma}
\begin{proof}
Lem. \ref{lemma:safe_prob_of_RCI} guarantees RCI tube invariance with probability $(1-\delta)$ for the sagittal CoM states such that $\state\local(t)\in\Omega(\state^\star,t) \iff \| \state\local(t) - \state^{\rm loc,\star}(t)\|_2 \leq \bar{\epsilon}(t)$ for all $t\in[0,T_{\rm step}]$. 
During the discrete transition, the state variation at the end of the step, $\delta x(T_{\rm step}^-)\doteq\state\local(T_{\rm step}^-)-\state^{\rm loc, \star}(T_{\rm step}^-)$, and the RCI tube bound $\bar{\epsilon}(T_{\rm step}^-)$ are both propagated by the saltation matrix as follows:
\begin{equation}
    \nonumber\delta x(0^+) = \Xi \, \delta x(T_{\rm step}^-), \quad
    \bar{\epsilon}(0^+) = \Xi \, \bar{\epsilon}(T_{\rm step}^-).
\end{equation}
If $\state\local (T_{\rm step}^-)\in\Omega(\state^{\rm loc, \star},T_{\rm step}^-)$ during the current walking step, then $\| \delta x(T_{\rm step}^-)\|\leq \bar{\epsilon}(T_{\rm step}^-)$. Since the saltation matrix is expansive with $\| \Xi\|\geq1$, we guarantee that the bound in the next walking step satisfies
\begin{equation}
    \| \Xi\delta x(T_{\rm step}^-)\| \leq \Xi\bar{\epsilon}(T_{\rm step}^-) \to
    \| \delta x(0^+)\| \leq \bar{\epsilon}(0^+).
\end{equation}
Therefore, $\state\local(0^+)\in\Omega(\state^{\rm loc, \star},0^+)$ holds for the next walking step. Since $\state\local(0^+)$ is initialized inside the RCI tube, Lem. \ref{lemma:safe_prob_of_RCI} ensures that it will remain invariant within the tube for the remainder of the step under bounded disturbance set $\mathcal{W}_{\rm terrain}$ 
with probability of $(1-\delta)$. For $i$ number of walking steps, where each step is an independent event since the true foot height is re-measured after the step, the tube invariance probability compounds according to the product rule, $\mathbb{P}\!\left( \bigcup_{i} \ (\  \state_i\local(t)\in\Omega(\state_i^{\rm loc,\star},t) ; \forall t \in[0,T_{\rm step}] )\right) = (1-\delta)^i$ for $i\in[1,...,H-1]$.
\end{proof}

\subsubsection{Contraction-based controller for full-order dynamics}\label{sec:CCM_for_fullorder}
Throughout the paper, the RCI tube $\Omega(\state^{\rm loc, \star},t)$ and the CCM control law $\tau_y^{\rm CCM}(t)$ are derived by the Augmented LIP dynamics. Here, we provide a formal proof that the contraction analysis of \ref{sec:RCIforBiped} can be extended to stabilize the full-order robot dynamics.

\looseness-1To realize full-order dynamics, we employ the ALIP planner\footnote{Since our ROM-based planner and controller are designed based on the sagittal dynamics. The lateral dynamics are considered at the lower level using the Angular Momentum LIP model \cite{gibson2022terrain}, since these motions are periodic given a fixed desired lateral foot placement.}~\cite{gibson2022terrain} to generate desired foot placements, together with a passivity-based controller~\cite{Sadeghian2017passivity} for the low-level control, which achieves asymptotic tracking at the joint level. Ankle actuation is incorporated within the passivity-based controller to enhance the ROM tracking performance~\cite{shamsah2023tamp}. To this end, we implement CCM control law as an flywheel torque applied directly at the robot’s CoM within the MuJoCo simulation environment~\cite{todorov2012mujoco}. Nonetheless, model mismatch between the ROM-based trajectory and the full-order dynamics persists due to imperfections in low-level control. We therefore state the following assumption regarding the model error.

\begin{assum}\label{assum:bound_model_error}
    \looseness-1Under a low-level controller with sufficiently stable tracking performance, the sagittal dynamics discrepancy $\disturb_{\rm model}$ between the ROM and the full-order model of the bipedal robot is bounded and belongs to $\mathcal{W}_{\rm model} \doteq \{ \disturb_{\rm model} \mid \| \disturb_{\rm model}\|\leq \bar{\disturb}_{\rm model}\}$, as shown in \cite{xiong2022_3d_unactuate}. $\disturb_{\rm model}$ can also be treated as a disturbance in the linear system \eqref{eq:AugLIP_dyn_CCM}.
\end{assum}

Similar to the approximation approach in \cite{xiong2022_3d_unactuate}, we estimate the full-order sagittal dynamics using the Aug-LIPM \eqref{eq:AugLIP_dyn_CCM}:
\begin{equation}\label{eq:full_order_approx}
    \dot{\state}^{\rm full} = A \state^{\rm full} + B  \tau_{y}^{\rm full} +  \disturb_{\rm terrain}+ \disturb_{\rm model},
\end{equation}
where $\state^{\rm full}=[x^{\rm loc,full},v^{\rm loc,full}]^\top \in \mathbb{R}^2$ is the true sagittal CoM states of the robot, $\disturb_{\rm terrain} \in \mathcal{W}_{\rm terrain}$
and $\disturb_{\rm model} \in \mathcal{W}_{\rm model}$ is treated as the bounded model discrepencies between ROM and full-order dynamics under Assumption \ref{assum:bound_model_error}. 

Without loss of generality, we can set $\state^{\rm loc, \star}$ (the Aug-LIPM reference to be tracked) as $\state^{\rm LIPM}$, since $\ctrl^{\rm loc, \star}=\tau_y^\star=0$.  We can leverage the CCM torque \eqref{eq:ccm_torque} as the full-order sagittal control input, denoted in a gain matrix form as $\tau_y^{\rm full}=\tau_y^{\rm CCM} = K(\state^{\rm full} - \state^{\rm LIPM})$, where $K = -\frac{1}{2}\rho B^\top M$. Then, the error between the full-order and ROM-based sagittal trajectory $e \doteq \state^{\rm full} -\state^{\rm LIPM}$ evolves under the closed-loop error dynamics: $ \dot{e} = (A + BK)e + \disturb_{\rm terrain} + \disturb_{\rm model}$. 

Since $A+BK$ is exponentially stable via the CCM, the error converges to a disturbance-invariant set $\Theta$: if $e_0 \in \Theta$, then $e(t) \in \Theta \ , \forall t\geq0$. By Lem.~\ref{lemma:safe_prob_of_RCI}--\ref{lemma:RCI_multiple_steps}, the RCI tube ensures that the error between the full-order and ROM-based trajectories caused by $\disturb_{\rm terrain}$ remains in $\Omega(\state^{\rm loc, \star},t)$ for all $t\geq0$ and across all steps. Since disturbances in~\eqref{eq:full_order_approx} are additive, Assumption~\ref{assum:bound_model_error} guarantees that the error due to $\disturb_{\rm model}$ also remains invariant within $\Omega(\state^{\rm loc, \star},t)$ by adjusting the bound $\bar{\epsilon}(t)$ with respect to the combined disturbance $\bar{\disturb}=\bar{\disturb}_{\rm terrain}+\bar{\disturb}_{\rm model}$.

\looseness-1In this work, we do not account for $\bar{\disturb}_{\rm model}$ in constructing the RCI tube. While this may slightly degrade tracking performance of the control law~\eqref{eq:ccm_torque}, we avoid overly conservative bounds so the entire RCI tube lies within the stable region of the phase portrait (Fig.~\ref{fig:LIPM_dynamics}(b)). We later enforce a constraint in Sec.~\ref{sec:MPC} restricting the RCI tube to the stable domain of positive orbital energy. For the bipedal robot, maintaining positive orbital energy ensures the CoM trajectory remains bounded around foot placement rather than diverging, thereby preventing falls.

\subsection{Uncertainty-Informed Model Predictive Control}\label{sec:MPC}

Given a terrain map estimated by GP mean $\mu(x,y)$ and $\mathcal{C}$, we formulate a high-level planner via MPC with horizon $H$ using global dynamics \eqref{eq:lip_dynamics} with state $\state=(x,y,z,v\local,\theta)$ and control $\ctrl = (u^f, u^{\Delta \theta})$ as follows:
\begin{subequations}
\begin{align} 
\min_{\state_{0:H}, \ctrl_{0:H-1}} \; & \| p_N - p_{\mathcal{G}} \|_{W_1}^2 + \| \theta_N - \theta_{\mathcal{G}} \|_{W_2}^2 \\[-7pt]
 &+ \textstyle\sum_{q = 0}^{H-1} \| \nabla_{x,y} \ \mu(x_q, y_q) \|_{W_3}^2 \nonumber
\\
\textrm{s.t.}\quad\  \; & \state _{q+1} = \Phi(\state _{q}, \bm{u} _{q}), \forall q\in[1,H-1] \label{general_mpc_dynamics}\\
    & \state_0 = \state_{\rm init}, \;  (\state _{q}, \bm{u} _{q}) \in \mathcal{XU} _{q}, \forall q    \\
    & c(z_{q}, z_{q+1}) \geq L\mathcal{C} \label{eq:mpc_cp_cons}\\
    &  z_0=\outp^{\rm true},  z_q=\mu(x\currq, y\currq) ,\forall q\\
    &  E(x^{\rm loc, \star}\currq(s), v^{\rm loc,\star}\currq(s), \bar{\epsilon}\currq(s)) > 0 ,\nonumber \\
    & \quad\quad\forall s\in[0, T_{\rm step}], \forall q, \label{eq:mpc_orbital_cons}
\end{align}
\end{subequations}
\looseness-1where $\mathcal{XU}_q = \{ 
(\state_q, \bm{u}_q) \mid 
\state_{\mathrm{lb}} \leq \state_q \leq \state_{\mathrm{ub}}$, $\bm{u}_{\mathrm{lb}} \leq \bm{u}_q \leq \bm{u}_{\mathrm{ub}} 
\}$.
The MPC cost is adopted from \cite{shamsah2024terrainaware}, which penalizes the Euclidean distance between the terminal state (i.e., the final position $p_N$ and the heading $\theta_N$) and the goal (i.e., the goal location $p_{\mathcal{G}}$ and the angle between the current position and the goal $\theta_{\mathcal{G}}$) and accumulates penalties over the norm of the terrain traversal slope $\nabla_{x,y} \ \mu(x_q, y_q)$ to avoid regions with steeper terrain. \eqref{eq:mpc_cp_cons} represents the CP safe footstep constraint, where $z_0 = \outp^{\rm true}$ and $z_q=\mu(x\currq, y\currq) ,\forall q\in\{1,\hdots,H-1\}$, implying that the true terrain height is known only at the stance foot of the current walking step for every iterations. $W_1$, $W_2$, and $W_3$ are weighting matrices.
\begin{remark}\label{remark:CP}
From Lem.~\ref{lemma:CP_safe}, \eqref{eq:mpc_cp_cons} ensures safe footstep selection with probability at least $1-\delta$ at the first step, consistent with receding-horizon MPC. For the next $(q+i)$ steps where $i \in \{1,\ldots,H-1\}$, Lem.~\ref{lemma:RCI_multiple_steps} ensures a safety probability $\mathbb{P}\!\left(c(z_q,z_{q+1})>0 \;\cup\; \cdots \;\cup\; c(z_{q+i},z_{q+i+1})>0\right) = (1-\delta)^i,$
since foot height is re-measured at each step, making safety events independent. The choice of $\delta$ for different scenarios is discussed in Sec. \ref{sec:results}.
\end{remark}

The constraint \eqref{eq:mpc_orbital_cons} enforces that the orbital energy $E$ of the continuous sagittal phase-space trajectory during the $q^{\rm th}$ step, $x^{\rm loc,\star}\currq(s), v\currq^{\rm loc,\star}(s)$ governed by \eqref{eq:sagittal_cons} given input $u_q^f$ and initial condition $x_q\local, v_q\local$ given by MPC, together with the RCI tube bounds $\bar{\epsilon}\currq(s)$ parameterized by $s \in [0,T_{\rm step}]$, remains positive. Note that $x^{\rm loc,\star}\currq(s), v\currq^{\rm loc,\star}(s)$ is also the desired reference trajectory in our contraction analysis. As seen in Fig. \ref{fig:LIPM_dynamics}(b), positive orbital energy $E>0$ ensures the robot's CoM has sufficient forward momentum to escape the gravitational pull of its current stance foot and move onto the next step~\cite{xiong2019orbitchar}. Graphically, this ensures that the RCI tube never crosses the asymptotic slope line defined in the phase-space plot. Finally, to enforce \eqref{eq:mpc_orbital_cons} over multiple walking steps within the horizon $H$, we apply the saltation matrix $\Xi$ to propagate the tube bounds from $\bar{\epsilon}\currq(T_{\rm step}^-)$ to $\bar{\epsilon}\nextq(0^+)$.
\section{Results}\label{sec:results}

\begin{figure*}[t]
    \centering
    \includegraphics[width=\linewidth]{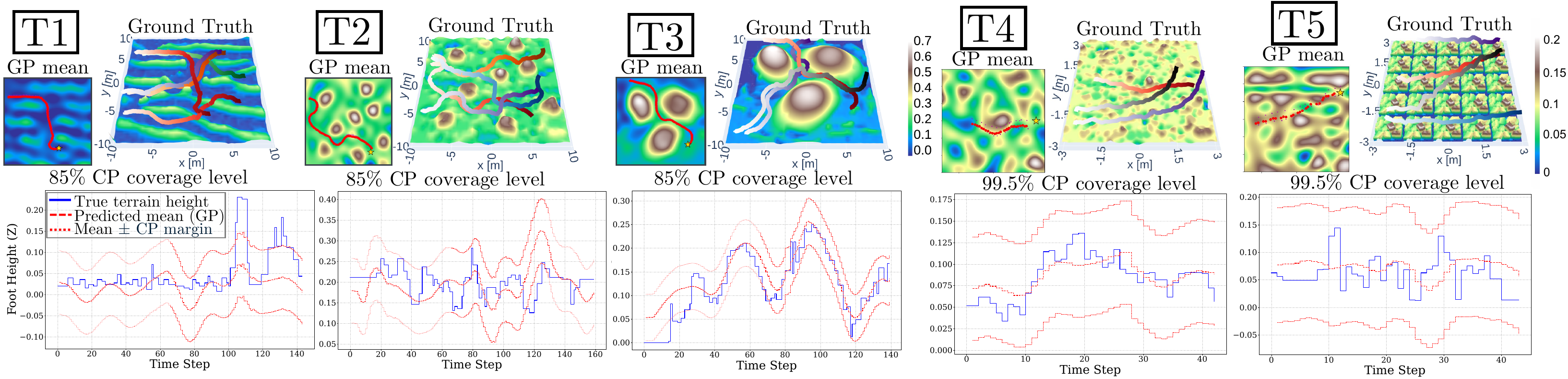}
    \caption{Visualization of three large terrain profiles, T1: bumpy and rough, T2: wavy and coarse, and T3: hilly with a smoother surface (for the case study with the 85\% confidence level), and two smaller terrain profiles, T4 and T5: level but rough (for the case study with the 99.5\% confidence level) with sample trajectories from different runs (white indicates the start of each trajectory). For the red trajectory, we show the GP mean estimate map along with the CP confidence interval, which ensures coverage of the true terrain height at each confidence level.}
    \label{fig:terrain_visual}
\end{figure*}

\begin{table}[t]
\centering
\caption{Simulation Results (85\% confidence level)}
\begin{tabular}{|l|c|c|c|}
\hline\label{table:CP85}
\textbf{Framework / Terrain} & \textbf{T1} & \textbf{T2} & \textbf{T3} \\
\hline
\multicolumn{4}{|c|}{\makecell{\textbf{Average Norm Error (across all steps)} \\ $\| \state\local(t)-\state^{\rm loc, \star}(t) \|_{\rm AVG}$}} \\
\hline
Ours & 0.040 & 0.045 & 0.039 \\
Ours w/o $\tau_{y}^\star$ & Fail & Fail & 0.051 \\
Baseline MPC\cite{shamsah2024terrainaware} & Fail & Fail & 0.052 \\
\hline
\multicolumn{4}{|c|}{\makecell{\textbf{Prob. of staying within the RCI Tube}\\ P($\| \state\local(t)-\state^{\rm loc, \star}(t) \| \leq \bar{\epsilon}(t) \ ; \forall t\in[0,T_{\rm step}])  $}} \\
\hline
Ours & 84.94\% & 77.09\% & 77.53\% \\
Ours w/o $\tau_{y}^\star$ & Fails & Fails & 58.03 \% \\
Baseline MPC\cite{shamsah2024terrainaware} & Fails & Fails & 54.75\% \\
\hline
\end{tabular}
\\
\vspace{0.05in}
{\scriptsize 'Fail' denotes Digit falls, reaching the goal in fewer than $5$ of $15$ trajectories. \par}
\end{table}

\begin{table}[t]
\centering
\caption{Performance on Terrain T3}
\begin{tabular}{|c|c|c|c|} 
 \hline \label{table:T3_analysis}
 \textbf{Method} & \textbf{Sim Time} & \textbf{\# of Steps} & \textbf{CP Coverage} \\
 \hline
 Ours & 62.30 & 142.33 & 85.85\% \\
 Ours w/o $\tau_{y}^\star$     & 66.32 & 149.13 & 80.69\% \\
 Baseline MPC\cite{shamsah2024terrainaware}   & 69.75 & 154.60 & 81.62\% \\
 \hline
\end{tabular}
\label{table:terrain3}
\end{table}

\begin{table}[t]
\centering
\caption{Simulation Results (99.5\% confidence level)}
\begin{tabular}{|l|c|c|}
\hline\label{table:results_T4T5}
\textbf{Framework / Terrain} & \textbf{T4} & \textbf{T5} \\
\hline
\multicolumn{3}{|c|}{\makecell{ $\| \state\local(t)-\state^{\rm loc, \star}(t) \|_{\rm AVG}$}} \\
\hline
Ours & 0.031 & 0.034 \\
Baseline MPC\cite{shamsah2024terrainaware} & 0.051 & 0.058 \\
\hline
\multicolumn{3}{|c|}{\makecell{P($\| \state\local(t)-\state^{\rm loc, \star}(t) \| \leq \bar{\epsilon}(t) \ ; \forall t\in[0,T_{\rm step}])  $}} \\
\hline
Ours & 93.60\% & 92.41\% \\
Baseline MPC\cite{shamsah2024terrainaware} & 65.03\% & 61.42\% \\
\hline
\end{tabular}
\end{table}

\looseness-1We evaluate the framework in MuJoCo simulations of the Digit biped across three  $20,\text{m}\times20,\text{m}$ terrains (T1–T3, height 0–0.7 m; Fig.~\ref{fig:terrain_visual}). Results at 85\% safety appear in Table~\ref{table:CP85}, with T3 analyzed further in Table~\ref{table:T3_analysis}. We also test smaller $6\,\text{m}\times6\,\text{m}$ terrains T4 and T5 (height from 0 to 0.2 m), where improved GP accuracy, due to a smaller input domain, enables 99.5\% confidence guarantees over multiple steps (Table~\ref{table:results_T4T5}). Our MPC framework, with and without CCM torque control, is compared against the baseline in~\cite{shamsah2024terrainaware} with only traversal slope minimization. For each terrain and framework, 15 trajectories with an average of 140 steps are simulated between pre-defined start and goal positions, using full-order dynamics in MuJoCo. All runs were performed on a laptop with an Intel i7 CPU and 16 GB RAM.

\looseness-1From 2500 discretized map points, 700 are randomly sampled with a 70-30 split (490 training, 210 calibration). At 85\% confidence, the average CP threshold across three terrains is $\mathcal{C}_{\rm AVG}=0.078$ while the GP bound is $0.360$ (i.e., $\pm1.44\sigma(\inp)_{\rm AVG}$), and for 99.5\% confidence, $\mathcal{C}_{\rm AVG}=0.060$ while GP bound is $0.708$ (i.e., $\pm2.81\sigma(\inp)_{\rm AVG}$). Thus, GP model coupled with CP method yields less conservative bounds, while preserving high confidence level. Table~\ref{table:CP85} reports the average norm error (ANE) between $\state\local(t)$ and $\state^{\rm loc,\star}(t)$, measured at $200$ Hz across trials, along with the probability of RCI tube invariance, computed as  
$(\sum_{q=0}^{N}\mathds{1}^{q}_{\Omega-\text{invariance}})/N$,  
where $\mathds{1}^{q}_{\Omega\text{-invariance}}=1$ if $\|\state\local(t)-\state^{\rm loc,\star}(t)\|\leq\bar{\epsilon}(t)$ for all $t\in[0,T_{\rm step}]$, and $0$ otherwise. Table~\ref{table:T3_analysis} further details terrain T3 performance, including average simulation time, total steps, and CP coverage, defined as  
$(\sum_{q=0}^{N}\mathds{1}^{q}_{\text{true coverage}})/N$,  
with $\mathds{1}^{q}_{\text{true coverage}}=1$ if $\outp\currq^{\rm true}\in[\mu(\inp\currq)-\mathcal{C},\ \mu(\inp\currq)+\mathcal{C}]$ during the $q^{\rm th}$ step, and $0$ otherwise.

On terrains T1 and T2, the robot fails to reach the goal in fewer than 5 of 15 trials when using baseline MPC or omitting the CCM control law $\tau_y^{\rm CCM}$. On rough terrains with small, frequent height changes (T1) and steeper slopes (T2), $\tau_y^{\rm CCM}$ improves stability by correcting off-nominal centroidal angular momentum, whereas the baselines lack recovery from sudden terrain variations and thus fail. On terrain T3, $\tau_y^{\rm CCM}$ reduces ANE relative to the baseline, yielding trajectories that track the MPC plan more closely, enabling more reliable goal reaching in fewer steps.

To validate Lem.~\ref{lemma:safe_prob_of_RCI}, we evaluate the probability of remaining within the RCI tube, expected at $85\%$ (i.e., $\state\local(t)$ would exit the tube only if the true terrain lies outside the CP interval). As noted in Sec.~\ref{sec:CCM_for_fullorder}, the torque control $\tau_y^{\rm CCM}$ may suffer slight performance loss from model mismatch. On terrain T1, with bumpy ground and minimal sagittal mismatch, $\tau_y^{\rm CCM}$ maintains the CoM within the tube at the target probability. On terrains T2 and T3, where sloped paths amplify sagittal and vertical perturbations, $\tau_y^{\rm CCM}$ still preserves tube invariance under terrain disturbances and mismatch with only an 8\% reduction from the desired 85\%, markedly outperforming the baseline or without $\tau_y^{\rm CCM}$.

\looseness-1On terrain T3, both simulation time and step count are reduced compared to the baseline, regardless of whether $\tau_y^{\rm CCM}$ is applied. This improvement arises from the CP constraint in~\eqref{eq:mpc_cp_cons}, which enables risk-seeking footstep planning by allowing exploration of regions with poor GP elevation estimates, provided true terrain variations lie within CP coverage intervals. Coupled with $\tau_y^{\rm CCM}$, our framework thus yields footstep plans that ensure safe terrain transitions with a level of confidence 85\% at every step.

Finally, we evaluate our framework on smaller terrains (T4 and T5) at a higher confidence level of 99.5\%. By Lem.~\ref{lemma:RCI_multiple_steps}, this guarantees RCI tube invariance for about 20 steps with compounded safety probability $(99.5\%)^{20}\approx 90.5\%$, strongly guaranteeing that the robot will track the full, safe MPC plan. As shown in Table~\ref{table:results_T4T5}, our framework achieves over 92\% per-step tube invariance and significantly lower ANE than the baseline. Thus, for short-horizon locomotion on immediate terrain with sparse data, GP with CP provides practical probabilistic guarantees for the full trajectory, and when combined with CCM torque control, enables reliable tracking of desired motion plans with high confidence.
\section{Conclusion}
We present a planning and control framework for bipedal navigation over uncertain terrain with probabilistic safety guarantees. Two core components of this framework include constructing a contraction-based reachable tube around the CoM trajectory to ensure safe tracking under terrain uncertainty and designing a flywheel torque control law to stabilize CoM angular momentum. Our framework ensures safe traversal with high confidence per walking step over large terrains with long horizons, and guarantees high-confidence safety across the full trajectory in smaller terrain maps with shorter horizons.

\bibliographystyle{IEEEtran}
\bibliography{main}

\end{document}